\documentclass{article}
\usepackage{ijcai19}

\usepackage{times}
\usepackage{soul}
\usepackage{url}
\usepackage[hidelinks]{hyperref}
\usepackage[utf8]{inputenc}
\usepackage[small]{caption}
\usepackage{graphicx}
\usepackage{amsmath}
\usepackage{booktabs}
\usepackage{algorithm}
\usepackage{algorithmic}
\usepackage{subfigure}

\usepackage{amssymb,amsthm}
\usepackage{multirow}
\usepackage{amsfonts}
\usepackage{bm}
\usepackage{xcolor}

\newtheorem{theorem}{Theorem}
\title{Attributed Graph Clustering via Adaptive Graph Convolution}

\author{
Xiaotong Zhang$^\ast$
\and
Han Liu$^\ast$\and
Qimai Li\footnote{indicates equal contribution.}\And
Xiao-Ming Wu\thanks{Corresponding author.}
\affiliations
Department of Computing, The Hong Kong Polytechnic University, Hong Kong
\emails
zxt.dut@hotmail.com,
liu.han.dut@gmail.com,
\{csqmli,csxmwu\}@comp.polyu.edu.hk
}

\begin{document}

\maketitle

\begin{abstract}
Attributed graph clustering is challenging as it requires joint modelling of graph structures and node attributes. Recent progress on graph convolutional networks has proved that graph convolution is effective in combining structural and content information, and several recent methods based on it have achieved promising clustering performance on some real attributed networks. However, there is limited understanding of how graph convolution affects clustering performance and how to properly use it to optimize performance for different graphs. Existing methods essentially use graph convolution of a fixed and low order that only takes into account neighbours within a few hops of each node, which underutilizes node relations and ignores the diversity of graphs. In this paper, we propose an adaptive graph convolution method for attributed graph clustering that exploits high-order graph convolution to capture global cluster structure and adaptively selects the appropriate order for different graphs. We establish the validity of our method by theoretical analysis and extensive experiments on benchmark datasets. Empirical results show that our method compares favourably with state-of-the-art methods.
\end{abstract}

\section{Introduction}

Attributed graph clustering \cite{tkde/CaiZC18} aims to cluster nodes of an attributed graph where each node is associated with a set of feature attributes. Attributed graphs widely exist in real-world applications such as social networks, citation networks, protein-protein interaction networks, etc. Clustering plays an important role in detecting communities and analyzing structures of these networks. However, attributed graph clustering requires  joint modelling of graph structures and node attributes to make full use of available data, which presents great challenges.

Some classical clustering methods such as $k$-means only deal with data features. In contrast, many graph-based clustering methods \cite{schaeffer2007graph} only leverage graph connectivity patterns, e.g., user friendships in social networks, paper citation links in citation networks, and genetic interactions in protein-protein interaction networks. Typically, these methods learn node embeddings using Laplacian eigenmaps \cite{newman2006finding}, random walks \cite{perozzi2014deepwalk}, or autoencoder \cite{wang2016structural}. Nevertheless, they usually fall short in attributed graph clustering, as they do not exploit informative node features such as user profiles in social networks, document contents in citation networks, and protein signatures in protein-protein interaction networks.

In recent years, various attributed graph clustering methods have been proposed, including methods based on generative models \cite{chang2009relational}, spectral clustering \cite{xia2014robust}, random walks \cite{yang2015network}, nonnegative matrix factorization \cite{wang2016semantic}, and graph convolutional networks (GCN) \cite{kipf2016semi}. In particular, GCN based methods such as GAE \cite{kipf2016variational}, MGAE \cite{wang2017mgae}, ARGE \cite{pan2018adversarially} have demonstrated state-of-the-art performance on several attributed graph clustering tasks.

Although graph convolution has been shown very effective in integrating structural and feature information, there is little study of how it should be applied to maximize clustering performance. Most existing methods directly use GCN as a feature extractor, where each convolutional layer is coupled with a projection layer, making it difficult to stack many layers and train a deep model. In fact, ARGE \cite{pan2018adversarially} and MGAE \cite{wang2017mgae} use a shallow two-layer and three-layer GCN respectively in their models, which only take into account neighbours of each node in two or three hops away and hence may be inadequate to capture global cluster structures of large graphs. Moreover, all these methods use a fixed model and ignore the diversity of real-world graphs, which can lead to suboptimal performance.

To address these issues, we propose an adaptive graph convolution (AGC) method for attributed graph clustering. The intuition is that neighbouring nodes tend to be in the same cluster and clustering will become much easier if nodes in the same cluster have similar feature representations. To this end, instead of stacking many layers as in GCN, we design a $k$-order graph convolution that acts as a low-pass graph filter on node features to obtain smooth feature representations, where $k$ can be adaptively selected using intra-cluster distance. AGC consists of two steps: 1) conducting $k$-order graph convolution to obtain smooth feature representations; 2) performing spectral clustering on the learned features to cluster the nodes. AGC enables an easy use of high-order graph convolution to capture global cluster structures and allows to select an appropriate $k$ for different graphs.
Experimental results on four benchmark datasets including three citation networks and one webpage network show that AGC is highly competitive and in many cases can significantly outperform state-of-the-art methods.

\section{Related Work}

Graph-based clustering methods can be roughly categorized into two branches: structural graph clustering and attributed graph clustering. Structural graph clustering methods only exploit graph structures (node connectivity). Methods based on graph Laplacian eigenmaps \cite{newman2006finding} assume that nodes with higher similarity should be mapped closer. Methods based on matrix factorization \cite{cao2015grarep,nikolentzos2017matching} factorize the node adjacency matrix into node embeddings. Methods based on random walks \cite{perozzi2014deepwalk,grover2016node2vec} learn node embeddings by maximizing the probability of the neighbourhood of each node. Autoencoder based methods \cite{wang2016structural,cao2016deep,ye2018deep} find low-dimensional node embeddings with the node adjacency matrix and then use the embeddings to reconstruct the adjacency matrix.

Attributed graph clustering \cite{yang2009combining} takes into account both node connectivity and features. Some methods model the interaction between graph connectivity and node features with generative models \cite{chang2009relational,yang2013community,he2017joint,bojchevski2018bayesian}. Some methods apply nonnegative matrix factorization or spectral clustering on both the underlying graph and node features to get a consistent cluster partition \cite{xia2014robust,wang2016semantic,li2018community,yang2015network}. Some most recent methods integrate node relations and features using GCN \cite{kipf2016semi}. In particular, graph autoencoder (GAE) and graph variational autoencoder (VGAE) \cite{kipf2016variational} learn node representations with a two-layer GCN and then reconstruct the node adjacency matrix with autoencoder and variational autoencoder respectively. Marginalized graph autoencoder (MGAE) \cite{wang2017mgae} learns node representations with a three-layer GCN and then applies marginalized denoising autoencoder to reconstruct the given node features. Adversarially regularized graph autoencoder (ARGE) and adversarially regularized variational graph autoencoder (ARVGE) \cite{pan2018adversarially} learn node embeddings by GAE and VGAE respectively and then use generative adversarial networks to enforce the node embeddings to match a prior distribution.

\section{The Proposed Method}

\subsection{Problem Formulation}

Given a non-directed graph $\mathcal{G}=(\mathcal{V},\mathcal{E}, X)$, where $\mathcal{V}=\{v_1, v_2, ..., v_n\}$ is a set of nodes with $|\mathcal{V}|=n$, $\mathcal{E}$ is a set of edges that can be represented as an adjacency matrix $A=\{a_{ij}\}\in \mathbb{R}^{n\times n}$, and $X$ is a feature matrix of all the nodes, i.e., $X=[\bm x_1, \bm x_2, \cdots, \bm x_n]^\top \in \mathbb{R}^{n\times d}$, where $\bm x_i \in \mathbb{R}^{d}$ is a real-valued feature vector of node $v_i$. Our goal is to partition the nodes of the graph $\mathcal{G}$ into $m$ clusters $\mathcal{C}=\{C_1, C_2, \cdots, C_m\}$.
Note that we call $v_j$ a $k$-hop neighbour of $v_i$, if $v_j$ can reach $v_i$ by traversing $k$ edges.

\begin{figure*}[t]
    \centering
    \subfigure[Frequency response]{\includegraphics[height=0.32\columnwidth, width=0.45\columnwidth]{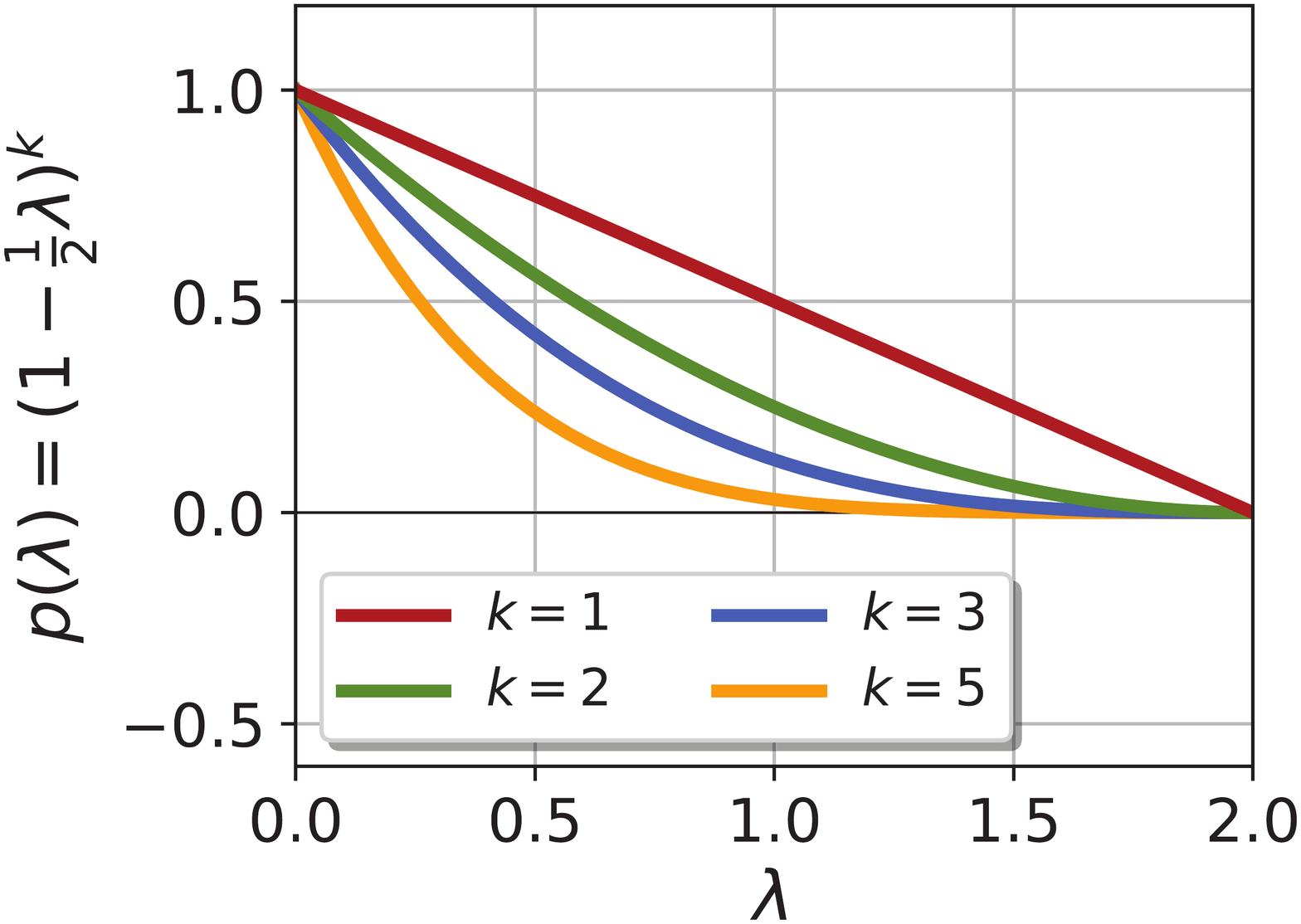}
    \label{fig1}}
    \hfil
    \subfigure[Raw features]{\includegraphics[height=0.32\columnwidth, width=0.38\columnwidth]{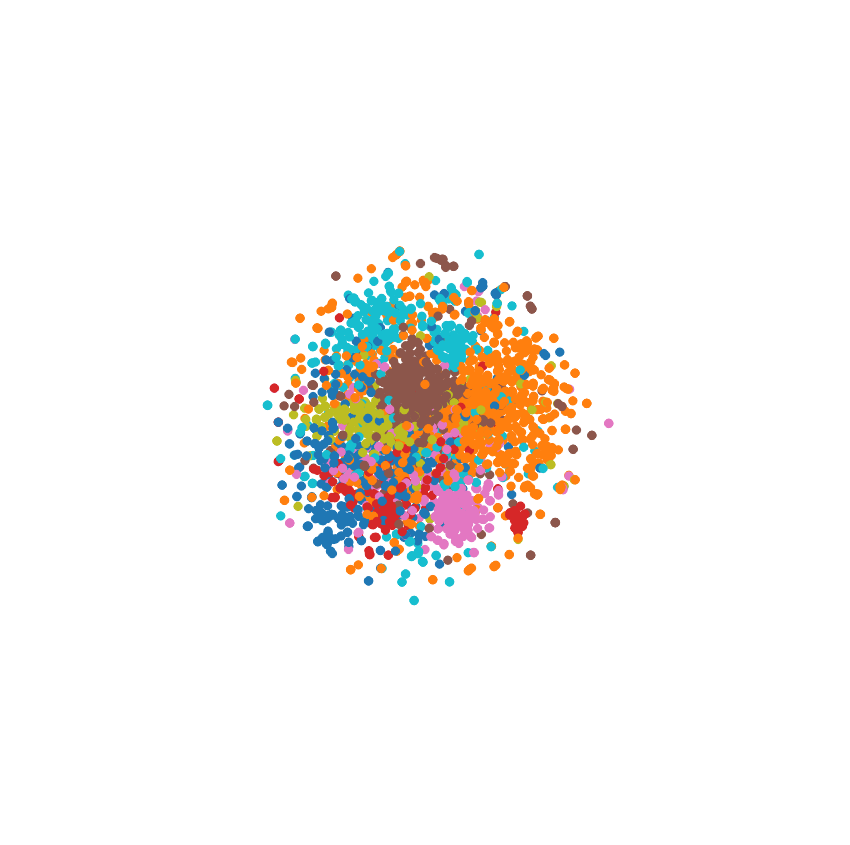}
    \label{fig2}}
    \hfil
    \subfigure[$k=1$]{\includegraphics[height=0.32\columnwidth, width=0.38\columnwidth]{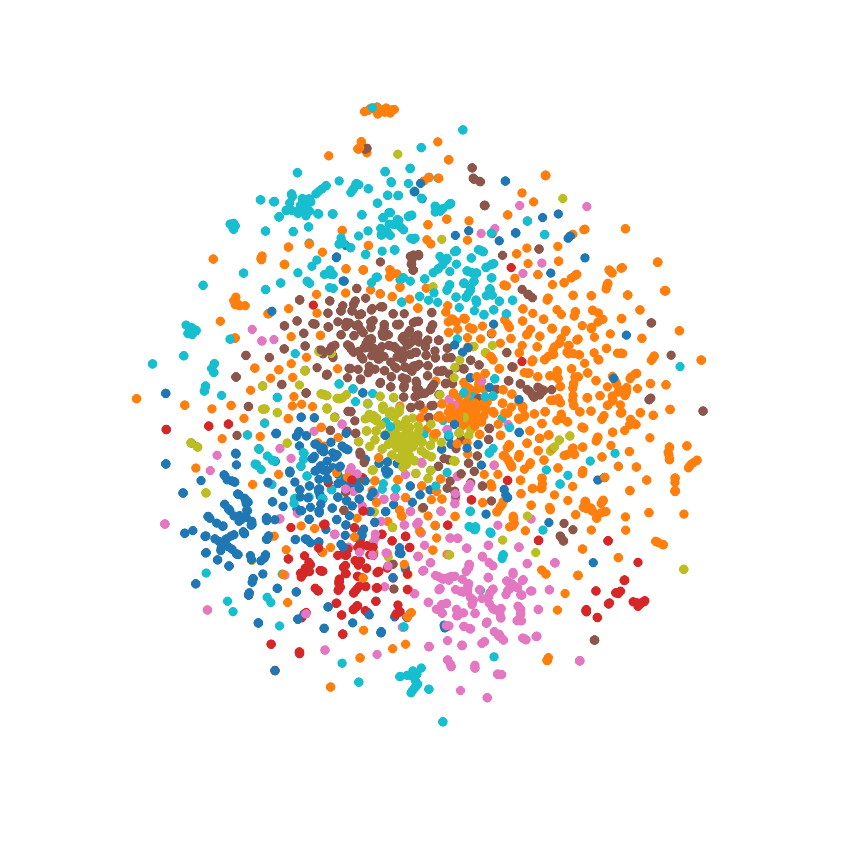}
    \label{fig2}}
    \hfil
    \subfigure[$k=12$]{\includegraphics[height=0.32\columnwidth, width=0.38\columnwidth]{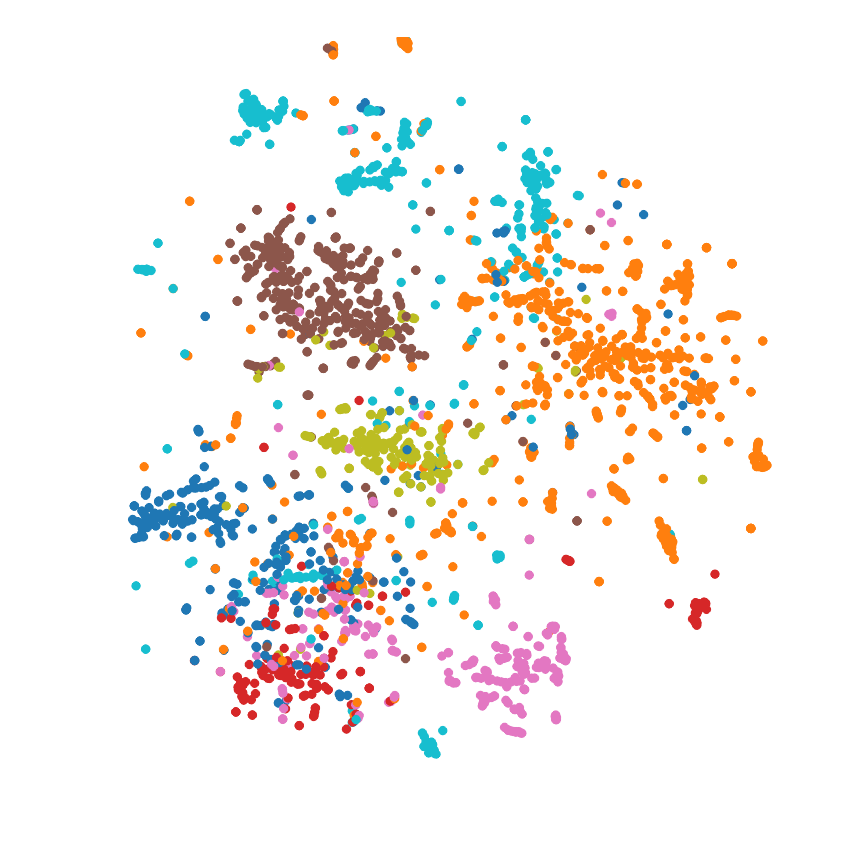}
    \label{fig3}}
    \hfil
    \subfigure[$k=100$]{\includegraphics[height=0.32\columnwidth, width=0.38\columnwidth]{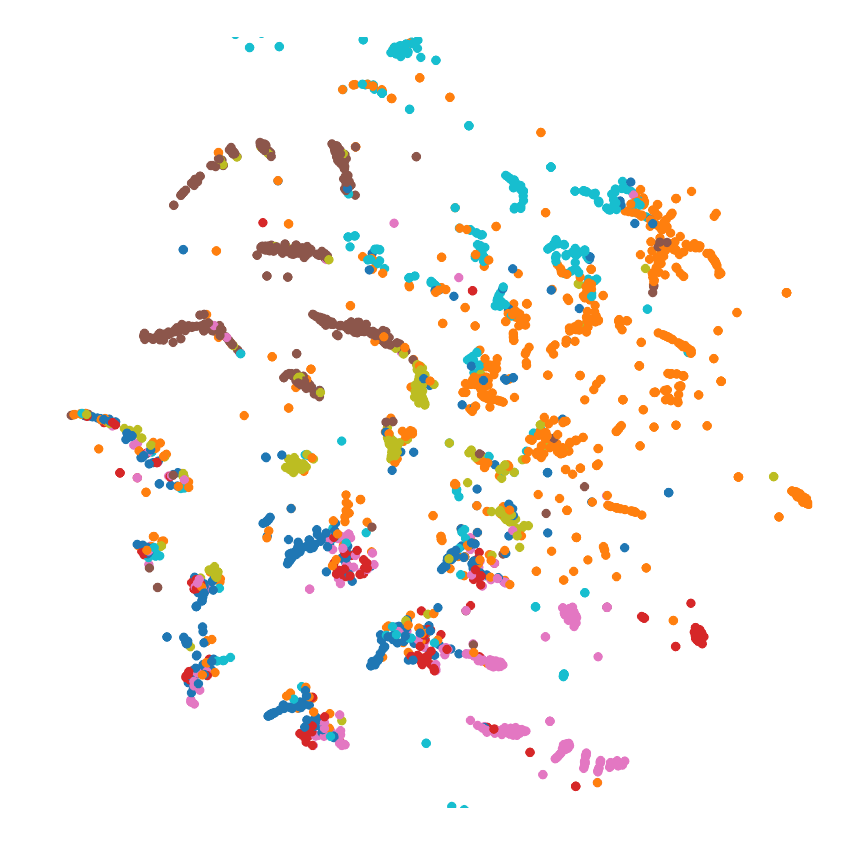}
    \label{fig4}}
    \caption{(a) Frequency response functions. (b-e) t-SNE visualization of the raw and filtered node features of Cora with different $k$.}
    \label{fig:frequency}
\end{figure*}

\subsection{Graph Convolution}
To formally define graph convolution, we first introduce the notions of graph signal and graph filter \cite{shuman2013emerging}.
A \emph{graph signal} can be represented as a vector $\bm f=[f(v_1),\cdots,f(v_n)]^\top$, where $f: \mathcal{V} \to \mathbb{R}$ is a real-valued function on the nodes of a graph.
Given an adjacency matrix $A$ and the degree matrix $D=\text{diag}(d_1, \cdots, d_n)$, the symmetrically normalized graph Laplacian $L_{s}=I-D^{-\frac12}AD^{-\frac12}$ can be eigen-decomposed as $L_s=U\Lambda U^{-1}$, where $\Lambda=\text{diag}(\lambda_1,\cdots,\lambda_n)$ are the eigenvalues in increasing order, and $U=[\bm u_1,\cdots,\bm u_n]$ are the associated orthogonal eigenvectors. A linear \emph{graph filter} can be represented as a matrix $G=U p(\Lambda)U^{-1}\in \mathbb{R}^{n\times n}$, where $p(\Lambda)=\text{diag}(p(\lambda_1),\cdots,p(\lambda_n))$ is called the frequency response function of $G$.
Graph convolution is defined as the multiplication of a graph signal $\bm f$ with a graph filter $G$:
\begin{equation}\label{conv}
  \bar{\bm f}=G \bm f,
\end{equation}
where $\bar{\bm f}$ is the filtered graph signal.

Each column of the feature matrix $X$ can be considered as a graph signal. In graph signal processing \cite{shuman2013emerging}, the eigenvalues $(\lambda_q)_{1\le q \le n}$ can be taken as frequencies and the associated eigenvectors $(\bm u_q)_{1\le q \le n}$ are considered as Fourier basis of the graph. A graph signal $\bm f$ can be decomposed into a linear combination of the eigenvectors, i.e.,
\begin{equation}\label{eq:signal-decomp}
    \bm f=U \bm z =\sum\nolimits_{q=1}^n{z_q \bm u_q},
\end{equation}
where $\bm z = [z_1,\cdots, z_n]^\top$ and $z_q$ is the coefficient of $\bm u_q$. The magnitude of the coefficient $|z_q|$ indicates the strength of the basis signal $\bm u_q$ presented in $\bm f$.

A graph signal is smooth if nearby nodes on the graph have similar features representations. The smoothness of a basis signal $\bm u_q$ can be measured by Laplacian-Beltrami operator $\Omega(\cdot)$ \cite{Chung97}, i.e.,
\begin{equation} \label{eq:laplacian_property}
\begin{split}
\Omega(\bm u_q)
&=\frac12\sum_{(v_i,v_j)\in \mathcal{E}}{a_{ij}\left\|\frac{\bm u_q(i)}{\sqrt{d_{i}}}-\frac{\bm u_q(j)}{\sqrt{d_{j}}}\right\|_2^2}\\
&= {\bm u_q}^\top L_s{\bm u_q} =\lambda_q,
\end{split}
\end{equation}
where $\bm u_q(i)$ denotes the $i$-th element of the vector $\bm u_q$. (\ref{eq:laplacian_property}) indicates that the basis signals associated with lower frequencies (smaller eigenvalues) are smoother, which means that a smooth graph signal $\bm f$ should contain more low-frequency basis signals than high-frequency ones. This can be achieved by performing graph convolution with a low-pass graph filter $G$, as shown below.

By (\ref{eq:signal-decomp}), the graph convolution can be written as
\begin{equation}\label{eq:filter_decom}
\bar{\bm f}= G\bm f =U p(\Lambda) U^{-1} \cdot U \bm z=\sum\nolimits_{q=1}^n{p(\lambda_q)z_q \bm u_q}.
\end{equation}
In the filtered signal $\bar{\bm f}$, the coefficient $z_q$ of the basis signal $\bm u_q$ is scaled by $p(\lambda_q)$. To preserve the low-frequency basis signals and remove the high-frequency ones in $\bm f$, the graph filter $G$ should be low-pass, i.e., the frequency response function $p(\cdot)$ should be decreasing and nonnegative.

A low-pass graph filter can take on many forms. Here, we design a low-pass graph filter with the frequency response function
\begin{equation}\label{fre}
p(\lambda_q)=1-\frac12\lambda_q.
\end{equation}
As shown by the red line in Figure \ref{fig1}, one can see that $p(\cdot)$ in (\ref{fre}) is decreasing and nonnegative on $[0,2]$. Note that all the eigenvalues $\lambda_q$ of the symmetrically normalized graph Laplacian $L_s$ fall into interval $[0,2]$ \cite{Chung97}, which indicates that $p(\cdot)$ in (\ref{fre}) is low-pass. The graph filter $G$ with $p(\cdot)$ in (\ref{fre}) as the frequency response function can then be written as
\begin{equation}\label{fil}
G=U p(\Lambda)U^{-1}=U(I-\frac12\Lambda)U^{-1}=I-\frac12 L_s.
\end{equation}

By performing graph convolution on the feature matrix $X$, we obtain the filtered feature matrix:
\begin{equation}\label{GC}
  \bar{X}=GX,
\end{equation}
where $\bar{X}=[\bm {\bar{x}}_1, \bm {\bar{x}}_2, \cdots, \bm {\bar{x}}_n]^\top \in \mathbb{R}^{n\times d}$ is the filtered node features after graph convolution. Applying such a low-pass graph filter on the feature matrix makes adjacent nodes have similar feature values along each dimension, i.e., the graph signals are smooth. Based on the cluster assumption that nearby nodes are likely to be in the same cluster, performing graph convolution with a low-pass graph filter will make the downstream clustering task easier.

Note that the proposed graph filter in (\ref{fil}) is different from the graph filter used in GCN. The graph filter in GCN is $G=I-L_s$ with the frequency response function $p(\lambda_q)=1-\lambda_q$ \cite{Qimai19}, which is clearly not low-pass as it is negative for $\lambda_q \in (1,2]$ .

\subsubsection{$k$-Order Graph Convolution}

To make clustering easy, it is desired that nodes of the same class should have similar feature representations after graph filtering. However, the first-order graph convolution in (\ref{GC}) may not be adequate to achieve this, especially for large and sparse graphs, as it updates each node $v_i$ by the aggregation of its $1$-hop neighbours only, without considering long-distance neighbourhood relations. To capture global graph structures and facilitate clustering, we propose to use $k$-order graph convolution.

We define $k$-order graph convolution as
\begin{equation}\label{kGC}
  \bar{X}=(I-\frac12 L_s)^kX,
\end{equation}
where $k$ is a positive integer, and the corresponding graph filter is
\begin{equation}\label{kfilter}
G=(I-\frac12 L_s)^k=U(I-\frac12\Lambda)^kU^{-1}.
\end{equation}
The frequency response function of $G$ in (\ref{kfilter}) is
\begin{equation}\label{kfre}
p(\lambda_q)=(1-\frac12\lambda_q)^k.
\end{equation}
As shown in Figure \ref{fig1}, $p(\lambda_q)$ in (\ref{kfre}) becomes more low-pass as $k$ increases, indicating that the filtered node features $\bar{X}$ will be smoother.

The iterative calculation formula of $k$-order graph convolution is
\begin{equation}\label{iter}
\begin{split}
&\bm {\bar{x}}_i^{(0)}={\bm x_i}, \quad \bm {\bar{x}}_i^{(1)}=\frac12 \left(\bm {\bar{x}}_i^{(0)}+\sum_{(v_i,v_j)\in \mathcal{E}}\frac{a_{ij}}{\sqrt{d_id_j}} \bm {\bar{x}}_j^{(0)}\right), \cdots,\\
&\bm {\bar{x}}_i^{(k)}=\frac12 \left(\bm {\bar{x}}_i^{(k-1)}+\sum_{(v_i,v_j)\in \mathcal{E}}\frac{a_{ij}}{\sqrt{d_id_j}} \bm {\bar{x}}_j^{(k-1)}\right),
\end{split}
\end{equation}
and the final $\bm {\bar{x}}_i$ is $\bm {\bar{x}}_i^{(k)}$.

From (\ref{iter}), one can easily see that $k$-order graph convolution updates the features of each node $v_i$ by aggregating the features of its $k$-hop neighbours iteratively. As $k$-order graph convolution takes into account long-distance data relations, it can be useful for capturing global graph structures to improve clustering performance.

\subsubsection{Theoretical Analysis}
As $k$ increases, $k$-order graph convolution will make the node features smoother on each dimension. In the following, we prove this using the Laplacian-Beltrami operator $\Omega(\cdot)$ defined in (\ref{eq:laplacian_property}).
Denote by $\bm f$ a column of the feature matrix $X$, which can be decomposed as $\bm f=U \bm z$. Note that $\Omega(\beta\bm f)=\beta^2\Omega(\bm f)$, where $\beta$ is a scalar. Therefore, to compare the smoothness of different graph signals, we need to put them on a common scale. In what follows, we consider the smoothness of a normalized signal $\frac{\bm f}{\lVert \bm f \rVert_2}$, i.e.,

\begin{equation}
    \Omega\left(\frac{\bm f}{\lVert \bm f \rVert_2}\right)
    = \frac{\bm f^\top L_s \bm f}{\lVert \bm f \rVert_2^2}
    = \frac{\bm z^\top\Lambda \bm z}{\lVert \bm z\rVert_2^2}
    = \frac{\sum_{i=1}^n \lambda_i z_{i}^2}{\sum_{i=1}^n{z_i^2}}.
\end{equation}

\begin{theorem}
    If the frequency response function $p(\lambda)$ of a graph filter $G$ is nonincreasing and nonnegative for all $\lambda_i$, then for any signal $\bm f$ and the filtered signal $\bar{\bm f}=G\bm f$, we always have
    $$\Omega\left(\frac{\bar{\bm f}}{\lVert \bar{\bm f} \rVert_2}\right)\le\Omega\left(\frac{\bm f}{\lVert \bm f \rVert_2}\right).$$
\end{theorem}
\begin{proof}
    We first prove the following lemma by induction.
    The following inequality
    \begin{equation}\label{eq:lemma}
        T_c^{(n)}=\frac{\sum_{i=1}^n c_iT_i}{\sum_{i=1}^n c_i} \le \frac{\sum_{i=1}^n b_iT_i}{\sum_{i=1}^n b_i}=T_b^{(n)}
    \end{equation}
    holds, if $T_1\le\cdots\le T_n$ and $\frac{c_1}{b_1}\ge\cdots\ge\frac{c_n}{b_n}$ with $\forall c_i, b_i \ge 0$. It is easy to validate that it holds when $n=2$.

    Now assume that it holds when $n=l-1$, i.e., $T_c^{(l-1)}\le T_b^{(l-1)}$. Then, consider the case of $n=l$ and we have
    \begin{equation}
        \begin{split}
            \frac{\sum_{i=1}^l c_iT_i}{\sum_{i=1}^l c_i}
            &= \frac{\sum_{i=1}^{l-1} c_iT_i + c_l T_l}{\sum_{i=1}^{l-1} c_i + c_l}\\
            &= \frac{(\sum_{i=1}^{l-1}{c_i})T_c^{(l-1)} + c_l T_l}{\sum_{i=1}^{l-1} c_i + c_l}\\
            &\le \frac{(\sum_{i=1}^{l-1}{c_i})T_b^{(l-1)} + c_l T_l}{\sum_{i=1}^{l-1} c_i + c_l}.
        \end{split}
    \end{equation}
    Since $T_b^{(l-1)}=\frac{\sum_{i=1}^{l-1} b_iT_i}{\sum_{i=1}^{l-1} b_i} \le \frac{\sum_{i=1}^{l-1} b_iT_{l-1}}{\sum_{i=1}^{l-1} b_i}=T_{l-1}$, we have $T_b^{(l-1)}\le T_l$. Also note that $\frac{\sum_{i=1}^{l-1} c_i}{\sum_{i=1}^{l-1} b_i}\ge\frac{c_l}{b_l}$. Since the lemma holds when $n = 2$, we have
    \begin{equation}
        \begin{split}
            \frac{(\sum_{i=1}^{l-1}{c_i})T_b^{(l-1)} + c_l T_l}{\sum_{i=1}^{l-1} c_i + c_l}
            &\le \frac{(\sum_{i=1}^{l-1}{b_i})T_b^{(l-1)} + b_l T_l}{\sum_{i=1}^{l-1} b_i + b_l} \\
            &= \frac{\sum_{i=1}^l b_iT_i}{\sum_{i=1}^l b_i},
        \end{split}
    \end{equation}
    which shows that the inequality (\ref{eq:lemma}) also holds when $n=l$. By induction, the above lemma holds for all $n$.

    We can now prove Theorem 1 using this lemma. For convenience, we arrange the eigenvalues $\lambda_i$ of $L_s$ in increasing order such that $0\le\lambda_1\le\cdots\le\lambda_n$. Since $p(\lambda)$ is nonincreasing and nonnegative,  $p(\lambda_1)\ge\cdots\ge p(\lambda_n)\ge0$. Theorem 1 can then be proved with the above lemma by letting
    \begin{equation}
        T_i=\lambda_i,\quad b_i=z_i^2,\quad c_i=p^2(\lambda_i)z_i^2.
    \end{equation}
\end{proof}

Assuming that $\bm f$ and $\bar{\bm f}$ are obtained by $(k-1)$-order and $k$-order graph convolution respectively, one can immediately infer from Theorem 1 that $\bar{\bm f}$ is smoother than $\bm f$. In other words, $k$-order graph convolution will produce smoother features as $k$ increases. Since nodes in the same cluster tend to be densely connected, they are likely to have more similar feature representations with large $k$, which can be beneficial for clustering.

\subsection{Clustering via Adaptive Graph Convolution}

We perform the classical spectral clustering method \cite{perona1998,von2007tutorial} on the filtered feature matrix $\bar{X}$ to partition the nodes of $\mathcal{V}$ into $m$ clusters, similar to \cite{wang2017mgae}. Specifically, we first apply the linear kernel $K=\bar{X}\bar{X}^T$ to learn pairwise similarity between nodes, and then we calculate $W=\frac12(|K|+|K^\top|)$ to make sure that the similarity matrix is symmetric and nonnegative, where $|\cdot|$ means taking absolute value of each element of the matrix. Finally, we perform spectral clustering on $W$ to obtain clustering results by computing the eigenvectors associated with the $m$ largest eigenvalues of $W$ and then applying the \emph{k}-means algorithm on the eigenvectors to obtain cluster partitions.

The central issue of $k$-order graph convolution is how to select an appropriate $k$. Although $k$-order graph convolution can make nearby nodes have similar feature representations, $k$ is definitely not the larger the better. $k$ being too large will lead to over-smoothing, i.e., the features of nodes in different clusters are mixed and become indistinguishable. Figure \ref{fig:frequency}(b-e) visualizes the raw and filtered node features of the \emph{Cora} citation network with different $k$, where the features are projected by t-SNE \cite{van2008visualizing} and nodes of the same class are indicated by the same colour. It can be seen that the node features become similar as $k$ increases. The data exhibits clear cluster structures with $k=12$. However, with $k=100$, the features are over-smoothed and nodes from different clusters are mixed together.

To adaptively select the order $k$, we use the clustering performance metric -- internal criteria based on the information intrinsic to the data alone \cite{Aggarwal14}. Here, we consider intra-cluster distance ($intra(\mathcal{C})$ for a given cluster partition $\mathcal{C}$), which represents the compactness of $\mathcal{C}$:
\begin{equation} \label{intra}
intra(\mathcal{C})=\frac{1}{|\mathcal{C}|}\sum_{C \in \mathcal{C}}\frac{1}{|C|(|C|-1)}\sum_{\substack{v_i,v_j \in C,\\ v_i\neq v_j}}\|\bar{x}_i-\bar{x}_j\|_2.
\end{equation}
Note that inter-cluster distance can also be used to measure clustering performance given fixed data features, and a good cluster partition should have a large inter-cluster distance and a small intra-cluster distance. However, by Theorem 1, the node features become smoother as $k$ increases, which could significantly reduce both intra-cluster and inter-cluster distances. Hence, inter-cluster distance may not be a reliable metric for measuring the clustering performance w.r.t. different $k$, and so we propose to observe the variation of intra-cluster distance for choosing $k$.

Our strategy is to find the first local minimum of $intra(\mathcal{C})$ w.r.t. $k$. Specifically, we start from $k=1$ and increment it by 1 iteratively. In each iteration $t$, we first obtain the cluster partition $\mathcal{C}^{(t)}$ by performing $k$-order ($k=t$) graph convolution and spectral clustering, then we compute $intra(\mathcal{C}^{(t)})$. Once $intra(\mathcal{C}^{(t)})$ is larger than $intra(\mathcal{C}^{(t-1)})$, we stop the iteration and set the chosen $k=t-1$. More formally, consider $d\_intra(t-1)=intra(\mathcal{C}^{(t)})-intra(\mathcal{C}^{(t-1)})$, the criterion for stopping the iteration is $d\_intra(t-1)>0$, i.e., stops at the first local minimum of $intra(\mathcal{C}^{(t)})$. So, the final choice of cluster partition is $\mathcal{C}^{(t-1)}$. The benefits of this selection strategy are two-fold. First, it ensures finding a local minimum for $intra(\mathcal{C})$ that may indicate a good cluster partition and avoids over-smoothing. Second, it is time efficient to stop at the first local minimum of $intra(\mathcal{C})$.

\subsection{Algorithm Procedure and Time Complexity}

\begin{algorithm}[t]\footnotesize
\caption{AGC}
\label{alg}
\begin{algorithmic}[1]
\REQUIRE
Node set $\mathcal{V}$, adjacency matrix $A$, feature matrix $X$, and maximum iteration number $max\_iter$.
\ENSURE
Cluster partition $\mathcal{C}$.
\STATE Initialize $t=0$ and $intra(\mathcal{C}^{(0)})=+\infty$. Compute the symmetrically normalized graph Laplacian $L_s=I-D^{-\frac12}AD^{-\frac12}$, where $D$ is the degree matrix of $A$.
\REPEAT
\STATE Set $t=t+1$ and $k=t$.
\STATE Perform $k$-order graph convolution by Eq.~(\ref{kGC}) and get $\bar{X}$.
\STATE Apply the linear kernel $K=\bar{X}\bar{X}^T$, and calculate the similarity matrix $W=\frac12(|K|+|K^T|)$.
\STATE Obtain the cluster partition $\mathcal{C}^{(t)}$ by performing spectral clustering on $W$.
\STATE Compute $intra(\mathcal{C}^{(t)})$ by Eq.~(\ref{intra}).
\UNTIL{$d\_intra(t-1)>0$ or $t>max\_iter$}
\STATE Set $k=t-1$ and $\mathcal{C}=\mathcal{C}^{(t-1)}$.
\end{algorithmic}
\end{algorithm}

The overall procedure of AGC is shown in Algorithm \ref{alg}.
Denote by $n$ the number of nodes, $d$ the number of attributes, $m$ the number of clusters, and $N$ the number of nonzero entries of the adjacency matrix $A$. Note that for a sparse $A$, $N<<n^2$. The time complexity of computing $L_s$ in the initialization is $\mathcal{O}(N)$. In each iteration, for a sparse $L_s$, the fastest way of computing $k$-order graph convolution in (\ref{kGC}) is to left multiply $X$ by $I-\frac12 L_s$ repeatedly for $k$ times, which has the time complexity $\mathcal{O}(Ndk)$. The time complexity of performing spectral clustering on $\bar{X}$ in each iteration is $\mathcal{O}(n^2d+n^2m)$. The time complexity of computing $intra(\mathcal{C})$ in each iteration is $\mathcal{O}(\frac1m n^2d)$. Since $m$ is usually much smaller than $n$ and $d$, the time complexity of each iteration is approximately  $\mathcal{O}(n^2d+Ndk)$. If Algorithm \ref{alg} iterates $t$ times, the overall time complexity of AGC is $\mathcal{O}(n^2dt+Ndt^2)$. Unlike existing GCN based clustering methods, AGC does not need to train the neural network parameters, which makes it time efficient.

\section{Experiments}

\subsection{Datasets}
We evaluate our method AGC on four benchmark attributed networks. \emph{Cora}, \emph{Citeseer} and \emph{Pubmed} \cite{kipf2016variational} are citation networks where nodes correspond to publications and are connected if one cites the other. \emph{Wiki} \cite{yang2015network} is a webpage network where nodes are webpages and are connected if one links the other.
The nodes in \emph{Cora} and \emph{Citeseer} are associated with binary word vectors, and the nodes in \emph{Pubmed} and \emph{Wiki} are associated with tf-idf weighted word vectors. Table \ref{tab:dataset} summarizes the details of the datasets.

%
%
%

\begin{table}[t]
\centering
\footnotesize

\begin{tabular}{lrrrr}
    \toprule
    Dataset    & \#Nodes & \#Edges & \#Features  & \#Classes  \\
    \midrule
    Cora       & 2708   & 5429   & 1433    & 7   \\
    Citeseer   & 3327   & 4732   & 3703    & 6  \\
    Pubmed     & 19717  & 44338  & 500     & 3     \\
    Wiki       & 2405   & 17981  & 4973    & 17  \\
    \bottomrule
\end{tabular}
\caption{Dataset statistics.}
\label{tab:dataset}
\end{table}

\begin{table*}[t]\footnotesize
 \centering
  \begin{tabular}{l l rcl| rcl| rcl| rcl}
   \toprule
   Methods & Input & \multicolumn{3}{c}{{Cora}} & \multicolumn{3}{c}{{Citeseer}} & \multicolumn{3}{c}{{Pubmed}} & \multicolumn{3}{c}{{Wiki}}\\
   \cmidrule(lr){3-5} \cmidrule(lr){6-8} \cmidrule(lr){9-11} \cmidrule(lr){12-14}
   & & {Acc\%} & {NMI\%} & {F1\%}  & {Acc\%} & {NMI\%} & {F1\%}  & {Acc\%} & {NMI\%} & {F1\%} & {Acc\%} & {NMI\%} & {F1\%}\\
   \midrule
   {\emph{k}-means}  & Feature & 34.65 & 16.73 & 25.42 & 38.49 & 17.02 & 30.47  & 57.32 & 29.12 & 57.35 & 33.37 & 30.20  & 24.51  \\
   Spectral-f  & Feature & 36.26 & 15.09 & 25.64 & 46.23 & 21.19 & 33.70  & 59.91 & \textbf{32.55} & \textbf{58.61} & 41.28 & 43.99 & 25.20  \\
   \midrule
   Spectral-g & Graph &  34.19 & 19.49 & 30.17 &  25.91 & 11.84 & 29.48 &  39.74 & 3.46 & 51.97 & 23.58 & 19.28 & 17.21  \\
   DeepWalk & Graph & 46.74 & 31.75 & 38.06  & 36.15 & 9.66 & 26.70  & 61.86 & 16.71 & 47.06 &  38.46 & 32.38 & 25.74  \\
   DNGR & Graph   & 49.24 & 37.29 & 37.29  & 32.59 & 18.02 & 44.19 & 45.35 & 15.38 & 17.90 & 37.58 & 35.85 & 25.38  \\
   \midrule
   GAE & Both & 53.25 & 40.69 & 41.97 & 41.26 & 18.34 & 29.13  & 64.08 & 22.97 & 49.26 &  17.33 &  11.93 & 15.35   \\
   VGAE & Both & 55.95 & 38.45 & 41.50 & 44.38 & 22.71 & 31.88  & \textbf{65.48} & 25.09 & 50.95 &  28.67 & 30.28 &  20.49  \\
   MGAE & Both & 63.43 & \textbf{45.57} & 38.01 & \textbf{63.56} & \textbf{39.75} & 39.49 & 43.88 & 8.16 & 41.98 & \textbf{50.14} & \textbf{47.97} & \textbf{39.20}\\
   ARGE & Both & \textbf{64.00} & 44.90 & 61.90 & 57.30 & 35.00 & \textbf{54.60}   & 59.12 & 23.17 & 58.41 & 41.40 & 39.50 &  38.27  \\
   ARVGE & Both & 63.80 & 45.00 & \textbf{62.70} & 54.40 & 26.10 & 52.90  & 58.22 & 20.62 & 23.04 & 41.55 & 40.01 & 37.80   \\
   \midrule
   AGC & Both & \textbf{68.92} & \textbf{53.68} & \textbf{65.61} & \textbf{67.00} & \textbf{41.13} & \textbf{62.48} & \textbf{69.78} & \textbf{31.59} & \textbf{68.72} & \textbf{47.65} & \textbf{45.28}   &  \textbf{40.36} \\
   \bottomrule
  \end{tabular}
  \caption{Clustering performance.}
  \label{tab:results}
\end{table*}

\subsection{Baselines and Evaluation Metrics}
We compare AGC with three kinds of clustering methods. 1) Methods that only use node features: \emph{k}-means and spectral clustering (Spectral-f) that constructs a similarity matrix with the node features by linear kernel. 2) Structural clustering methods that only use graph structures: spectral clustering (Spectral-g) that takes the node adjacency matrix as the similarity matrix, DeepWalk \cite{perozzi2014deepwalk}, and deep neural networks for graph representations (DNGR) \cite{cao2016deep}. 3) Attributed graph clustering methods that utilize both node features and graph structures: graph autoencoder (GAE) and graph variational autoencoder (VGAE) \cite{kipf2016variational}, marginalized graph autoencoder (MGAE) \cite{wang2017mgae}, and adversarially regularized graph autoencoder (ARGE) and variational graph autoencoder (ARVGE) \cite{pan2018adversarially}.

To evaluate the clustering performance, we adopt three widely used performance measures \cite{Aggarwal14}: clustering accuracy (Acc), normalized mutual information (NMI), and macro F1-score (F1).

\subsection{Parameter Settings}
For AGC, we set $max\_iter=60$. For other baselines, we follow the parameter settings in the original papers. In particular, for DeepWalk, the number of random walks is 10, the number of latent dimensions for each node is 128, and the path length of each random walk is 80. For DNGR, the autoencoder is of three layers with 512 neurons and 256 neurons in the hidden layers respectively. For GAE and VGAE, we construct encoders with a 32-neuron hidden layer and a 16-neuron embedding layer, and train the encoders for 200 iterations using the Adam optimizer with learning rate 0.01.
For MGAE, the corruption level $p$ is 0.4, the number of layers is 3, and the parameter $\lambda$ is $10^{-5}$. For ARGE and ARVGE, we construct encoders with a 32-neuron hidden layer and a 16-neuron embedding layer. The discriminators are built by two hidden layers with 16 neurons and 64 neurons respectively. On \emph{Cora}, \emph{Citeseer} and \emph{Wiki}, we train the autoencoder-related models of ARGE and ARVGE for 200 iterations with the Adam optimizer, with the encoder learning rate and the discriminator learning rate both as 0.001; on \emph{Pubmed}, we train them for 2000 iterations with the encoder learning rate 0.001 and the discriminator learning rate 0.008.

\begin{figure}[t]
\centering
\subfigure[Cora]{\includegraphics[height=0.45\columnwidth, width=0.48\columnwidth]{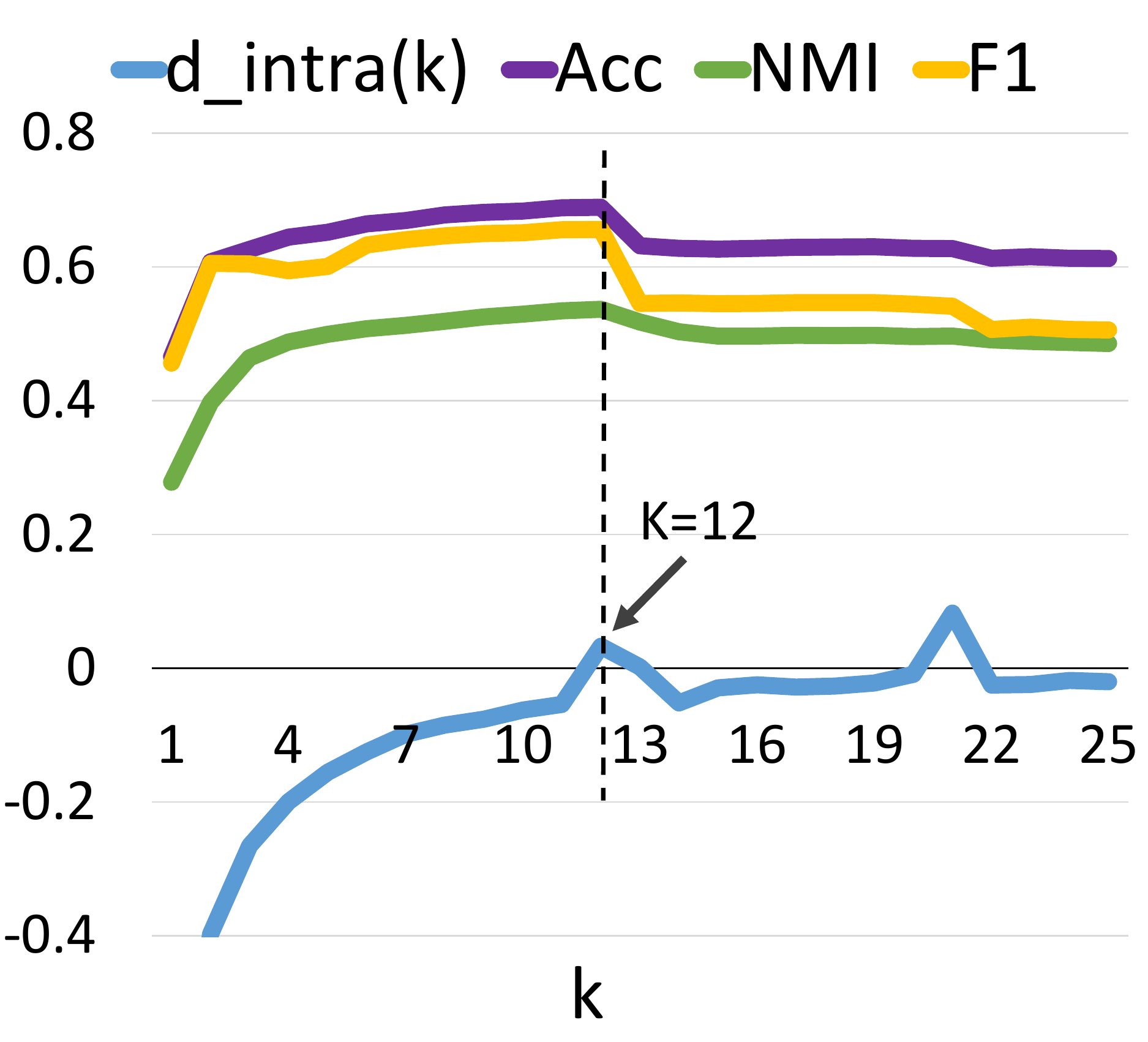}
\label{fig_curve1}}
\hfil
\subfigure[Wiki]{\includegraphics[height=0.45\columnwidth, width=0.48\columnwidth]{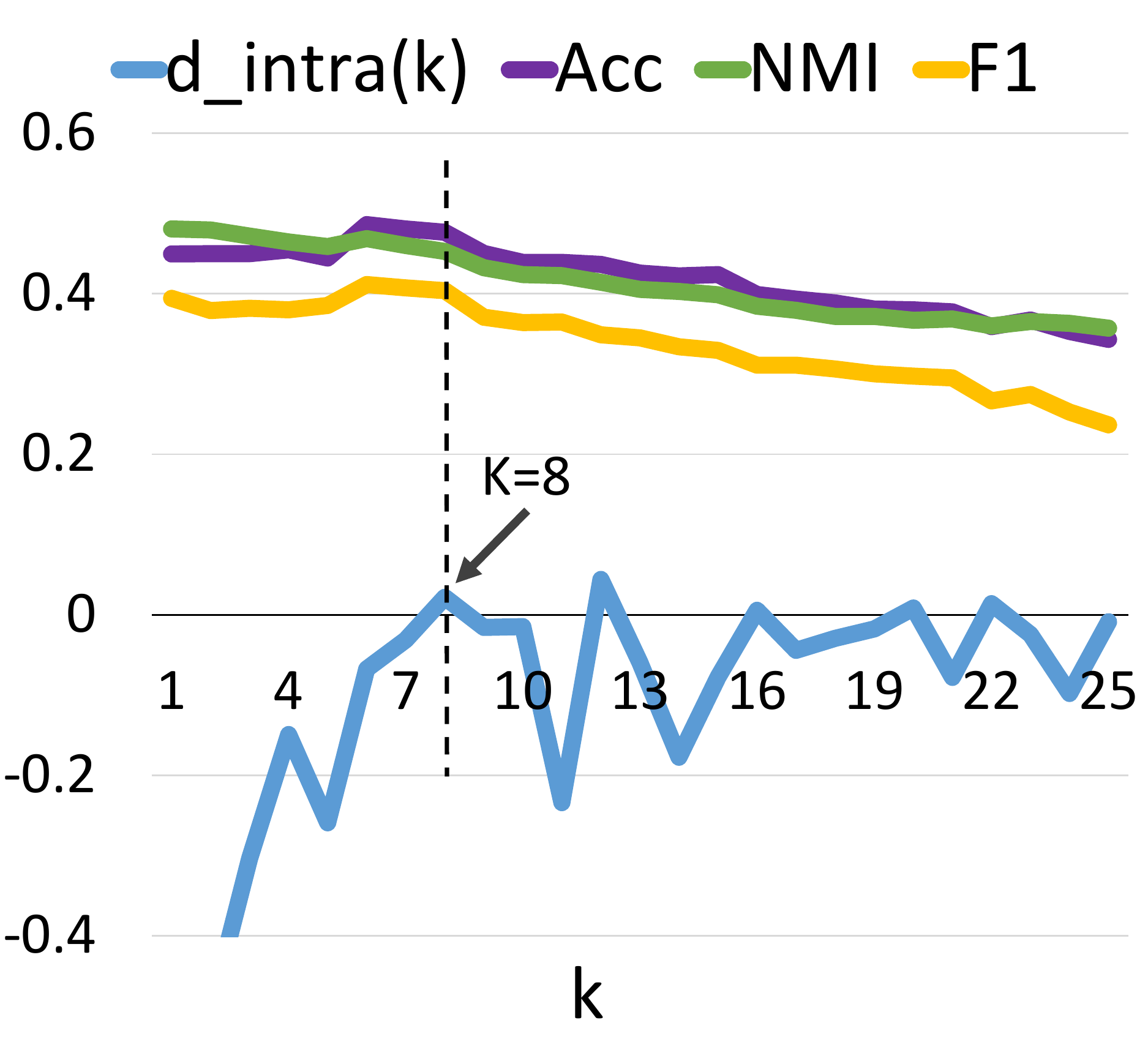}
\label{fig_curve2}}
\caption{$d\_intra(k)$ and clustering performance w.r.t. $k$.}
\label{fig:k}
\end{figure}

\subsection{Result Analysis}

We run each method 10 times for each dataset and report the average clustering results in Table \ref{tab:results}, where the top 2 results are highlighted in bold. The observations are as follows.

1) AGC consistently outperforms the clustering methods that only exploit either node features or graph structures by a very large margin, due to the clear reason that AGC makes a better use of available data by integrating both kinds of information, which can complement each other and greatly improve clustering performance.

2) AGC consistently outperforms existing attributed graph clustering methods that use both node features and graph structures. This is because AGC can better utilize graph information than these methods. In particular, GAE, VGAE, ARGE and ARVGE only exploit 2-hop neighbourhood of each node to aggregate information, and MGAE only exploits 3-hop neighbourhood. In contrast, AGC uses $k$-order graph convolution with an automatically selected $k$ to aggregate information within $k$-hop neighbourhood to produce better feature representations for clustering.

3) AGC outperforms the strongest baseline MGAE by a considerable margin on \emph{Cora}, \emph{Citeseer} and \emph{Pubmed}, and is comparable to MGAE on \emph{Wiki}. This is probably because \emph{Wiki} is more densely connected than others and aggregating information within 3-hop neighbourhood may be enough for feature smoothing. But it is not good enough for larger and sparser networks such as \emph{Citeseer} and \emph{Pubmed}, on which the performance gaps between AGC and MGAE are wider. AGC deals with the diversity of networks well via adaptively selecting a good $k$ for different networks.

To demonstrate the validity of the proposed selection criterion $d\_intra(t-1)>0$, we plot $d\_intra(k)$ and the clustering performance w.r.t. $k$ on \emph{Cora} and \emph{Wiki} respectively in Figure \ref{fig:k}. The curves of \emph{Citeseer} and \emph{Pubmed} are not plotted due to space limitations. One can see that when $d\_intra(k)>0$, the corresponding Acc, NMI and F1 values are the best or close to the best, and the clustering performance declines afterwards. It shows that the selection criterion can reliably find a good cluster partition and prevent over-smoothing. The selected $k$ for \emph{Cora}, \emph{Citeseer}, \emph{Pubmed} and \emph{Wiki} are 12, 55, 60, and 8 respectively, which are close to the best $k\in [0, 60]$ -- 12, 35, 60, and 6 on these datasets respectively, demonstrating the effectiveness of the proposed selection criterion.

AGC is quite stable. The standard deviations of Acc, NMI and F1 are 0.17\%, 0.42\%, 0.01\% on \emph{Cora}, 0.24\%, 0.36\%, 0.19\% on \emph{Citeseer}, 0.00\%, 0.00\%, 0.00\% on \emph{Pubmed}, and 0.79\%, 0.17\%, 0.20\% on \emph{Wiki}, all very small.

The running time of AGC (in Python, with NVIDIA Geforce GTX 1060 6GB GPU) on \emph{Cora}, \emph{Citeseer}, \emph{Pubmed} and \emph{Wiki} is 5.78, 62.06, 584.87, and 10.75 seconds respectively. AGC is a little slower than GAE, VGAE, ARGE and ARVGE on \emph{Citeseer}, but is more than twice faster on the other three datasets. This is because AGC does not need to train the neural network parameters as in these methods, and hence is more time efficient even with a relatively large $k$.

\section{Conclusion}

In this paper, we propose a simple and effective method for attributed graph clustering. To make a better use of available data and capture global cluster structures, we design a $k$-order graph convolution to aggregate long-range data information. To optimize clustering performance on different graphs, we design a strategy for adaptively selecting an appropriate $k$. This enables our method to achieve competitive performance compared to classical and state-of-the-art methods.
In future work, we plan to improve the adaptive selection strategy to make our method more robust and efficient.

\section*{Acknowledgments}

This research was supported by the grants 1-ZVJJ and G-YBXV funded by the Hong Kong Polytechnic University.

\bibliographystyle{named}
\bibliography{ijcai19}

\end{document}